\documentclass[accepted]{uai2025} 
                        
\usepackage[american]{babel}

\usepackage{authblk}
\usepackage{natbib} 
\usepackage{mathtools} 
\usepackage{booktabs} 
\usepackage{tikz} 
\usepackage{amsmath, amssymb, amsthm}
\usepackage{amsfonts}
\usetikzlibrary{arrows.meta, positioning, fit, shapes.multipart}
\usepackage{geometry}
\usepackage{algorithm, algpseudocode}
\usepackage{multirow} 
\usepackage{colortbl}
\usepackage{graphicx}
\usepackage{array}    
\usepackage{xcolor}

\title{On the Privacy Risks of Spiking Neural Networks: \\A Membership Inference Analysis}

\newtheorem{lemma}{Lemma}
\newtheorem{definition}{Definition}

\newtheorem{proposition}{Proposition}
\author{Junyi Guan\thanks{Equal Contribution}\thanks{Correspondence to: \texttt{junyi.guan@mbzuai.ac.ae}, \texttt{abhijith.sharma@mbzuai.ac.ae}}}
\author{Abhijith Sharma\protect\footnotemark[1]\protect\footnotemark[2]}
\author{Chong Tian}
\author{Salem Lahlou}
\affil{%
    Machine Learning Department\\
    Mohamed Bin Zayed University of Artificial Intelligence\\
    Abu Dhabi, UAE.
}

\begin{document}
\maketitle



\begin{abstract}
Spiking Neural Networks (SNNs) are increasingly explored for their energy efficiency and robustness in real-world applications, yet their privacy risks remain largely unexamined. In this work, we investigate the susceptibility of SNNs to Membership Inference Attacks (MIAs)—a major privacy threat where an adversary attempts to determine whether a given sample was part of the training dataset. While prior work suggests that SNNs may offer inherent robustness due to their discrete, event-driven nature, we find that its resilience diminishes as latency (T) increases. Furthermore, we introduce an input dropout strategy under black box setting, that significantly enhances membership inference in SNNs. Our findings challenge the assumption that SNNs are inherently more secure, and even though they are expected to be better, our results reveal that SNNs exhibit privacy vulnerabilities that are equally comparable to Artificial Neural Networks (ANNs). Our code is available at \url{https://github.com/sharmaabhijith/MIA_SNN}
\end{abstract}

\section{Introduction}

Spiking Neural Networks (SNNs) are a class of neural networks that emulate the discrete, event-driven mechanisms of biological neurons. Unlike Artificial Neural Networks (ANNs) that process continuous signals, SNNs communicate through discrete spikes, activating only when relevant information is present \citep{tavanaei2019deep}. The design of SNNs for handling complex tasks remains an active research area. However, their unique spiking operations impart certain inherent advantages across a wide range of practical applications \citep{nunes2022spiking, stan2024learning, baek2024snn}. In particular, advances in neuromorphic computing have elicited greater traction in the use of SNNs for sparse and efficient hardware implementations of neural networks \citep{yao2024spike, minhas2024continual}.

The primary adoption of SNNs is driven by two key advantages. First, the event-driven nature of spikes allows for reduced computations, leading to significant energy savings. Hence, SNNs naturally emerge as a favorable choice for energy constrainted domains, such as edge devices \citep{yan2024reconsidering}, Internet of Things (IoT) \citep{liu2022ultralow} and low-power embedded platforms \citep{syed2021exploring}. Second, SNNs have demonstrated resilience against various adversarial threats. Their discrete spike-based processing enables enhanced security by capturing only the relevant information, and obscuring less important patterns which are often exploited by adversaries \citep{nagarajan2022analysis}. 

Given these advantages, SNNs have emerged as a compelling choice for various real-world applications in recent years. Several studies have explored replacing ANN strategies with spiking dynamics, demonstrating their utility in domains such as image classification, object detection, and speech recognition \citep{yamazaki2022spiking}. Their ability to process information efficiently and robustly has led to their integration into industrial systems, enhancing performance and reliability \citep{wang2023brain, zhou2023fast}.

Simultaneously, with the rapid emergence of high-performance and efficient architectures \citep{lin2022survey, vaswani2017attention, brown2020language, sze2017efficient}, organizations worldwide are racing to adopt AI models across domains. As machine learning models are increasingly deployed in sensitive domains, data privacy has become a critical concern. Moreover, AI users are now more aware and apprehensive about potential data leaks \citep{jobin2019global}. This shift has sparked an unprecedented demand for AI service providers to safeguard the privacy of the data they manage. Unfortunately, numerous studies have exposed the vulnerabilities of current systems to privacy attacks, highlighting the urgent need for private models \citep{liu2021machine,rigaki2023survey}.

One of the most significant threats in this domain is Membership Inference Attacks (MIAs). In an MIA, an adversary aims to determine whether a specific data sample was part of a model's training set, potentially leading to privacy breaches \citep{hu2022membership}.
Studying MIAs allows researchers to identify vulnerabilities and design robust models to withstand the threat. For instance, in healthcare, ensuring that patient data cannot be inferred from a predictive model is vital for maintaining confidentiality \citep{hu2022membership}. In addition, regulating organizations can leverage MIA as a metric to conduct predeployment audits of AI models to avoid leakage of sensitive data.


Given the increasing deployment of SNNs in various applications \citep{chu2022neuromorphic, hussaini2024applications, zhou2023fast, bos2023sub}, it is imperative to assess their vulnerability to MIAs. Preliminary studies suggest that the discrete nature of SNN outputs may offer some protection against such attacks. However, existing research is limited, with only a few studies addressing this issue, often using metrics that may not fully capture the nuances of SNN behavior \citep{carlini2022membership}. Motivated by our observations, we aim to investigate: Are SNNs truly as robust to MIAs as we might hope, or do they share the same privacy risks as their counterparts that we have yet to uncover?

To address the privacy concerns of SNNs, our work systematically analyzes their resilience against MIAs through the following contributions:

\begin{itemize}
\item Develop efficient empirical estimates of state-of-the-art MIAs to simulate the online attack setting.
\item Propose a simple input dropout method to introduce stochasticity in the membership inference process, improving attack effectiveness in SNNs.
\item Analyze the impact of SNNs' latency and training methods on their vulnerability to MIAs, and also compare them with ANNs.
\item Investigate the feasibility of using ANNs to perform MIAs against SNNs, when attacker does not have the knowledge of SNNs' parameters.
\end{itemize}

To contextualize our contributions, we first review existing research on SNNs and MIAs, including concurrent work, in Section \ref{sec:related_work}. Section \ref{sec:prelim} provides a concise overview of the SNN dynamics, training, and MIA formulation. In Section \ref{sec:method}, we propose an input dropout-based strategy to efficiently evaluate SNNs against MIAs, while also comparing them with traditional ANNs. Finally, we validate our hypothesis through experiments and results in Section \ref{sec:exp_res}.

\section{Related work}
\label{sec:related_work}

\textbf{Spiking Neural Networks (SNNs)}: 
Designing and training SNNs is challenging due to their sensitivity to hyperparameters such as membrane threshold and synaptic latency, both of which significantly impact performance \citep{bouvier2019spiking}. Consequently, many existing methods focus on achieving low-latency inference with improved convergence while maintaining accuracy \citep{meng2022training}. Traditional approaches like surrogate gradient learning \citep{neftci2019surrogate} and temporal coding \citep{bellec2018long} have been further enhanced by advanced techniques \citep{dampfhoffer2023backpropagation}. Notably, \citet{deng2022temporal} introduced a gradient re-weighting mechanism to improve the temporal efficiency of SNN training.
To eliminate the need for manual threshold selection, \citet{bojkovic2024data} proposed a data-driven approach for threshold selection and potential initialization. Their method also facilitates the conversion of trained ANNs into SNNs, enabling efficient and high-performance training even in low-latency settings ($T=1,2,4$). These advancements are crucial for real-time applications. Beyond energy efficiency, SNNs have also been explored for their inherent robustness against adversarial \citep{nomura2022robustness,el2021securing} and model inversion attacks \citep{kim2022privatesnn}, further reinforcing their potential towards robust AI.

\textbf{Membership Inference Attacks: }The threat of Membership Inference Attacks (MIAs) was first demonstrated by \citet{shokri2017membership} in a simple Machine Learning-as-a-Service (MLaaS) black-box setting. Since then, extensive research has explored the privacy risks associated with diverse neural network architectures for a wide range of applications \citep{hu2022membership,yeom2018privacy,salem2019mlleaks}. Despite significant advancements and robustness characteristics of SNNs \citep{kim2022privatesnn}, their vulnerability to MIAs remains largely unclear and underexplored \citep{sharmin2019comprehensive}.

The inconsistencies in evaluation metrics and experimental settings in existing studies have made direct comparisons of MIA techniques challenging \citep{hu2023defenses}. However, \citet{carlini2022membership} presented MIA from first principles, emphasizing the importance of analyzing the Receiver Operating Characteristic (ROC) curve in attack's assessments. The ROC fully captures the tradeoff between True Positive Rate (TPR) and False Positive Rate (FPR) of the membership data across different classification thresholds. Reporting TPR under extremely low FPR conditions ($\leq$1\% and $\leq$ 0.1\%) is particularly crucial, as attackers prioritize confidently identifying members over overall accuracy. More recently, \citet{zarifzadeh2024low} proposed a state-of-the-art attack, called robust MIA (RMIA), and generalized all other existing MIAs under the umbrella of their attack formulation. RMIA also achieved highly effective attack performance with a limited number of shadow/reference models - auxiliary models trained on data with similar properties to the target model's training data.

\textbf{Concurrent Work: } While existing research primarily focuses on traditional ANNs, the membership privacy risks in SNNs remain largely unexamined. A recent study by \citet{li2024membership} explored the robustness of SNNs against MIAs, incorporating diverse experimental settings and assessing the impact of data augmentation. However, despite these contributions, the study suffers from several critical limitations. It relies on biased evaluation metrics such as balanced accuracy, which can obscure the true effectiveness of MIAs, and employs outdated training techniques for SNNs \citep{carlini2022membership}. Nowadays, in many research papers, AUC and TPR at very low FPR are the main metrics to study the performance of MIAs.
Additionally, the evaluation is conducted on simple datasets, failing to provide meaningful insights into real-world scenarios. 
More importantly, the study neglects key advancements in attack methodologies, such as RMIA, limiting the comprehensiveness of its findings. Furthermore, the analysis is restricted to timestep variations in neuromorphic datasets, lacking a systematic investigation of static datasets. These shortcomings underscore the insufficiency of existing efforts in rigorously assessing membership privacy risks in SNNs. A more sound evaluation is necessary to bridge this gap and uncover the true privacy vulnerabilities of SNNs.

\section{Preliminaries}
\label{sec:prelim}

\subsection{Spiking Neural Network}

\subsubsection{Data Encoding}
Unlike ANNs, SNNs have an additional temporal dimension to represent the data. This means that, in the case of image processing, the input to the network is not a single image, but is instead encoded over time into a sequence of T images, where T is known as \emph{time step} or \emph{latency}. There are various methods to encode images \citep{almomani2019comparative, wang2023comparison}, however, in our work, we use the classical \emph{constant encoding} that simply replicates the original image T times, and the resulting data is fed into the SNN.

\subsubsection{Forward Pass}

The forward process of SNN begins with encoding the input data. As the spikes are fed into the network, each neuron updates its membrane potential based on the spikes it receives. If a neuron's membrane potential exceeds a certain threshold, the neuron fires a spike and sends it to the next layer through synapses. The Equation \ref{eq:snn1} and \ref{eq:snn2} together describes the dynamics of a neuron in an SNN. 
\begin{equation}
\label{eq:snn1}
    U_{l}^{t} = \lambda U_{l}^{t-1} + w_{l}O_{l-1}^{t}
\end{equation}
\begin{equation}
\label{eq:snn2}
    O_{l}^{t}= H(U_{l}^{t}-\theta_{tr}),~~~~~~ U_{l}^{t} = U_{reset}
\end{equation}
Equation \ref{eq:snn1} represents the membrane potential \( U_{l}^{t} \) of a neuron at layer \( l \) and time \( t \). The membrane potential is updated by combining two components: the decayed potential from the previous time step \( \lambda U_{l}^{t-1} \), where \( \lambda \) is a decay factor between $0$ and $1$, and the weighted input from the previous layer's output \( w_{l}O_{l-1}^{t} \), where \( w_{l} \) is the synaptic weight and \( O_{l-1}^{t} \) is the output (spike) from the previous layer at time \( t \). If the decay factor \( \lambda \) is $1$, the neuron is called an \emph{Integrate-and-Fire (IF)} neuron; otherwise, it is a \emph{Leaky Integrate-and-Fire (LIF)} neuron.

The Equation \ref{eq:snn2} describes the neuron's output at time \( t \), where \( H \) is the Heaviside step function. If the membrane potential \( U_{l}^{t} \) exceeds a certain threshold \( \theta_{tr} \), the neuron generates an output spike (\( O_{l}^{t} = 1 \)); otherwise, no spike is generated (\( O_{l}^{t} = 0 \)). Once the neuron fires a spike, its membrane potential is reset to a specific value \( U_{reset} \), preparing the neuron for further spike processing in subsequent time steps. For the last layer $L$, the output of the network is $U_{L}^{T}$, which is the membrane potential for the final time $T$. 

\subsubsection{SNN Training}
\label{snn_training}

The training of SNNs can be categorized into two main approaches: direct backpropagation and ANN-to-SNN conversion (hybrid). However, in MIAs, the probability vector output by the neural network is the most critical feature for determining the membership status of a specific data point. Therefore, when studying the effectiveness of MIAs, we should focus more on the forward propagation process of SNNs rather than their backpropagation process.


Training SNNs directly from scratch can be challenging sometimes as the performance is sensitive to hyperparameters, like membrane threshold. On the other hand, SNN converted from trained ANNs along with hybrid training can help achieve good accuracy due to better initialization of threshold from the trained ANNs \citep{kim2022privatesnn}. For example, a recent work from \citet{bojkovic2024data} shows how the activation values from the layers of a trained ANN can be utilized to initialize the membrane threshold in SNN. This method coupled with hybrid training \citep{rathi2021enabling} enables achieving good performance even for low-latency SNNs, which is beneficial for diverse real-time applications.

\subsection{Membership Inference}

In binary classification, let \( D = \{(x_i, y_i)\}_{i=1}^{n} \) be a dataset where each \( x_i \in \mathbb{R} \) represents a one-dimensional feature value, and \( y_i \in \{0,1\} \) is the class label, where \( y_i = 1 \) indicates a positive sample and \( y_i = 0 \) indicates a negative sample. A threshold-based classifier with threshold \( t \) classifies a sample as positive if \( x_i \geq t \) and negative otherwise. For a given threshold \( t \), the true positive rate (TPR) and false positive rate (FPR) is defined as shown in Equation \ref{eq:tprfpr}

\begin{equation}
\label{eq:tprfpr}
\scriptsize
    \text{TPR}(t) = \frac{\sum\limits_{i=1}^{n} \mathbf{1} (x_i \geq t, y_i=1)}{\sum\limits_{i=1}^{n} \mathbf{1} (y_i=1)}, ~~~
    \text{FPR}(t) = \frac{\sum\limits_{i=1}^{n} \mathbf{1} (x_i \geq t, y_i=0)}{\sum\limits_{i=1}^{n} \mathbf{1} (y_i=0)}
\end{equation}

The Receiver Operating Characteristic (ROC) curve is the set of points \( \{(\text{FPR}(t), \text{TPR}(t)) \mid t \in \mathbb{R} \} \), and the Area Under the Curve (AUC) is given by the integral $\int_0^1 \text{TPR}(\text{FPR}) \, d(\text{FPR}).$

\begin{definition}[Membership Inference Game] Let $\pi$ be an underlying data distribution and let $\mathcal{A}$ be a training algorithm. 
The game proceeds between two entities: a \emph{challenger} and an \emph{attacker} \citep{carlini2022membership,zarifzadeh2024low}
\label{def:membership-inference-game}.

\begin{enumerate}
  \item \textbf{Dataset Sampling and Model Training.}  
  The challenger samples a dataset $S$ from $\pi$, and trains a model $\theta$ using $\mathcal{A}$ on $S$. Here, $\theta$ is the target model.
  
  \item \textbf{Challenge Sample Selection.}  
  The challenger tosses a fair coin $b \in \{0,1\}$. 
  \begin{itemize}
    \item If $b = 1$, it picks a point $x$ uniformly at random from $S$. In this case, $x$ is a member of $\theta$.
    \item If $b = 0$, it draws a sample $x$ from $\pi$ (ensuring $x \notin S$). In this case, $x$ is not a member of $\theta$.
  \end{itemize}
  The challenger then provides both the target model $\theta$ and the sample $x$ to the attacker.
  
  \item \textbf{Attacker's Inference.}  
  The attacker, having access to $\theta$ (and potentially query access to the distribution $\pi$), computes a \emph{membership score} 
  \textbf{$\mathrm{Score}_{\text{MIA}}(x;\theta)$} indicating how likely it believes $x$ was contained in $S$. Based on a chosen threshold $\beta$, the attacker issues a membership decision
  \[
    \hat{b} \;\leftarrow\; \mathbf{1}\bigl[\mathrm{Score}_{\text{MIA}}(x;\theta) \;\ge\; \beta\bigr].
  \]
  Here, $\mathbf{1}[\cdot]$ is the indicator function, hence, $\hat{b} = 1$ (member) if the score exceeds the threshold, and $0$ (non-member) otherwise.

  \item \textbf{Outcome and Evaluation.}  
  The attacker's prediction $\hat{b}$ is compared against the true bit $b$. Hence, membership is inferred by the attacker if $\hat{b} = b$. By repeating this experiment over many trials, we can plot the ROC for each $\beta$. The leakage of the model is often characterized by the achievable trade-off between TPR and FPR across all possible threshold values. The attack method is considered good if it achieves a high AUC score.
\end{enumerate}
\end{definition}

\section{Methodology} \label{sec:method}
In this section, we describe our methodology for evaluating SNNs against MIA. Our attack setup follows prior research \citep{zarifzadeh2024low, carlini2022membership}. In MIAs, an adversary uses shadow models (or reference models) to replicate the target model's behavior. These reference models are trained on the target model's data distribution. If a reference model's training data contains the target query, it is an \emph{IN model}; otherwise, it is an \emph{OUT model}.

A MIA is classified as an online attack if half of the reference models are \emph{IN model}, ensuring an unbiased estimate. However, this approach is computationally expensive since \emph{IN model} must be trained for each target query. In contrast, offline attacks use only \emph{OUT model} and apply a scaling factor as a hyperparameter to approximate the online setting. This creates a trade-off between attack strength and computational cost.


\begin{figure}[t]
    \centering
    \includegraphics[width=\linewidth]{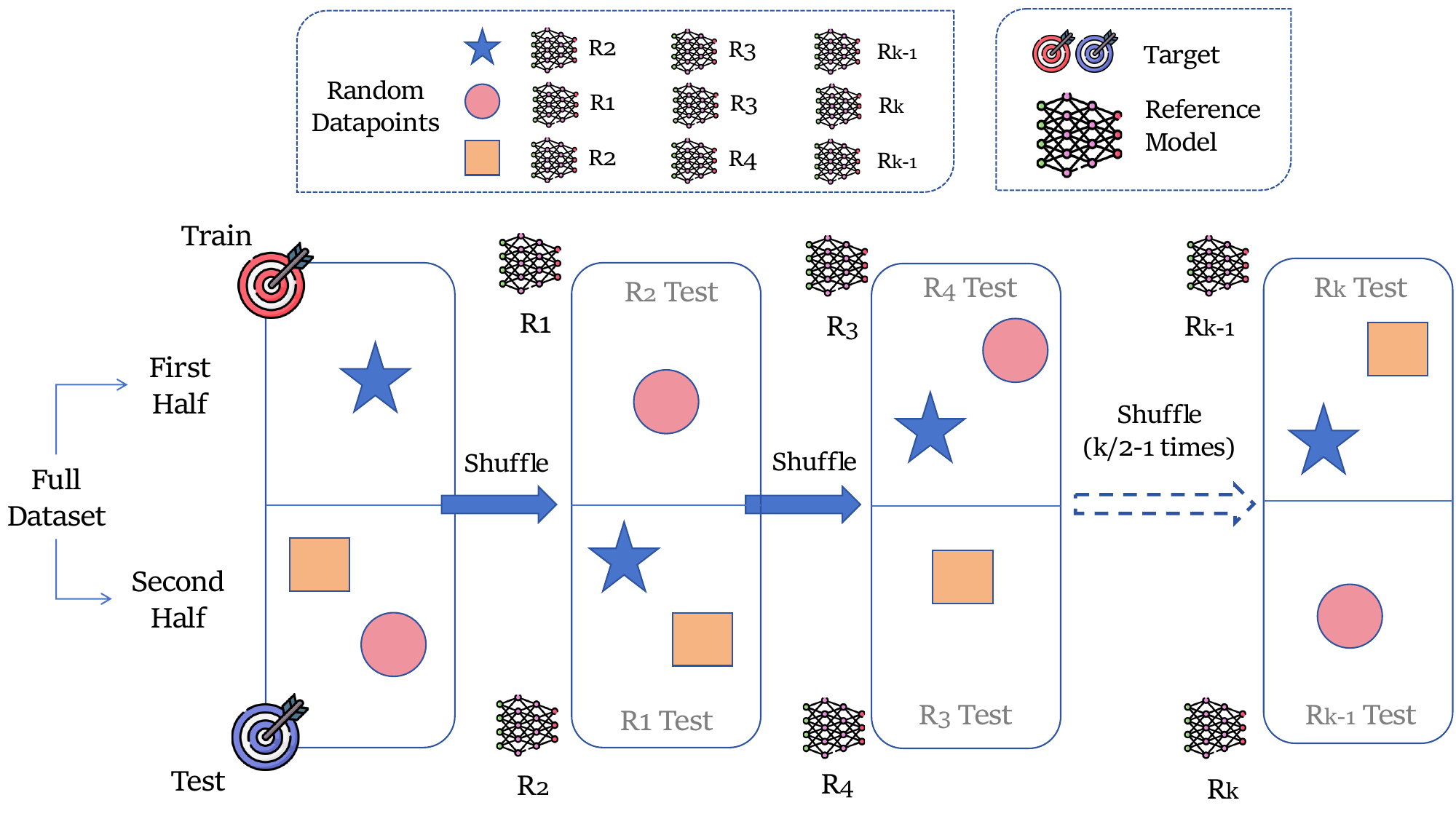}
    \caption{Illustration of Data Splitting Strategy}
    \label{fig:datasplit}
\end{figure}

To rigorously evaluate the vulnerability of SNNs to Membership Inference Attacks (MIAs), we prioritize \textbf{simulating an online attack scenario}. Online MIAs are preferred as they more realistically reflect real-world attack settings and are often more revealing of subtle vulnerabilities compared to offline approaches. To achieve this efficiently, we propose a data splitting strategy that emulates an online attack without incurring the computational cost of repeatedly retraining models. Our method begins by partitioning the dataset $D$ into two equal halves: one to train our target SNN model and the other to serve as its test set. Subsequently, we shuffle the original dataset $D$. This shuffled dataset is then iteratively divided to train $n$ pairs of reference models, resulting in a total of $2n$ reference models. Crucially, while each reference model is trained on a substantial portion of $D$ derived from this process, the data splitting is carefully designed to ensure that each query from the target model's test set is present in approximately half of the reference models. This balanced exposure is essential for unbiased MIA evaluation. 

This approach effectively simulates an online attack environment by providing a diverse set of pre-trained reference models against which to assess membership, all without the need for computationally expensive online model retraining. Consistent with standard MIA assumptions in the literature \citep{zarifzadeh2024low}, we acknowledge that our reference models are designed under the common assumption that attackers possess knowledge of the target model's architecture. Our approach is depicted in Figure \ref{fig:datasplit}, showcassing an efficient and fair data-splitting strategy tailored for online attack settings. The dataset \( D \) is initially balanced and split into training (\( Tr \)) and test (\( Ts \)) sets for the target model. By iteratively shuffling and re-splitting \( D \), a total of \( 2n \) reference models are trained. Each target query is included in exactly half of these reference models, ensuring a fair distribution and enabling robust evaluation. This method avoids redundant training and ensures computational efficiency while maintaining unbiased experimental results.

\subsection{MIA Attack Methods: Background and Modifications}

We start by describing the common MIA attacks: \emph{Attack-P, Attack-R, and RMIA} \citep{ye2022enhanced,zarifzadeh2024low} and adapt them to our experimental setting. These three attack methods exploit different assumptions about model confidence to infer the membership. The parameters for the attack methods can be defined as follows. Let $\theta$ be the target model and $\theta'$ a reference model. The audited sample is $x$, while $z$ is drawn from the dataset $D$ of size $N$. The model's confidence output is $\Pr(\cdot | \theta)$. Reference models trained with and without $x$ are denoted as $\theta'_x$ and $\theta'_{\bar{x}}$, respectively. The confidence of any model $m$, $\Pr(d | m)$, is given by the softmax score of the true label of input $d$.

Based on this setting, the original \emph{Attack-P} is defined in Table \ref{table:miadef}. Note that \emph{Attack-P} operates without reference models, assuming that if \(x\) belongs to \(\theta\)'s training data, the model's confidence in \(x\) is typically higher than for an arbitrary sample \(z\) from the overall distribution. The empirical estimation of the MIA score is given by:  

\begin{equation}  
    \frac{1}{N} \sum_{j=1}^{N} \mathbf{1}\Bigl( \Pr(x \mid \theta) \geq \Pr(z_j \mid \theta) \Bigr)  
\end{equation}  

However, since our setting provides access to the full dataset \(D\), we refine the MIA score by considering all \(x \in D\). Furthermore, by Proposition \ref{Theorem1}, we adopt \(\Pr(x | \theta)\) as the modified \emph{Attack-P} score, ensuring preservation of the ROC curve. While the weakest among the three, it provides insights into the model's confidence distribution.



\begin{table}[h]
    \centering
    
    \renewcommand{\arraystretch}{2.3}  
    \setlength{\tabcolsep}{4pt}  
    \rowcolors{3}{white}{teal!15}
    \scalebox{0.78}{
    \begin{tabular}{lcc}
        \toprule
        \textbf{\large Attack} & \textbf{\large Original} & \textbf{\large Modified} \\
        \midrule
        \emph{\large Attack-P} &
        $\displaystyle \Pr_z \left( \frac{\Pr(x | \theta)}{\Pr(z | \theta)} \geq 1 \right)$ &
        $\displaystyle \Pr(x | \theta)$ \\
        \emph{\large Attack-R} &
        $\displaystyle \Pr_{\theta'} \left( \frac{\Pr(x | \theta)}{\Pr(x | \theta')} \geq 1 \right)$ &
        $\displaystyle \frac{1}{2n} \sum_{\theta'} \mathbf{1} \left( \Pr(x | \theta) \geq \Pr(x | \theta') \right)$ \\
        \emph{\large RMIA} &
        $\displaystyle \Pr_z \left( \frac{\Pr(\theta | x)}{\Pr(\theta | z)} \geq 1 \right)$ &
        $\displaystyle \frac{\Pr(x | \theta)}{\frac{1}{2n} \left( \sum_{\theta'_x} \Pr(x | \theta'_x) + \sum_{\theta'_{\bar{x}}} \Pr(x | \theta'_{\bar{x}}) \right)}$ \\
        \bottomrule
    \end{tabular}
    }
    \caption{Attack formulations for computational efficiency by removing dependence on $z$.}
    \label{table:miadef}
\end{table}

\emph{Attack-R} builds on \emph{Attack-P}'s approach, but with the help of reference models. Unlike \emph{Attack-P}, the \emph{Attack-R} compares \( \theta \)'s confidence on \( x \) with that of the reference models (\( \theta' \)), assuming that the target model assigns higher confidence to samples it has seen during training. Similar to the \emph{Attack-P}'s empirically estimated score, the modified \emph{Attack-R} can be computed as shown in Table \ref{table:miadef}. 

Finally, \emph{RMIA} also known as robust-MIA is the strongest of all. It exploits the assumption that if \( \theta \) is trained on \( x \), the likelihood of \( \theta \) given \( x \) is higher than for any unrelated sample \( z \). Similar to \emph{Attack-P}, by Proposition \ref{Theorem1}, one can easily find the equivalence between the original and the modified \emph{RMIA}. 
Moreover, with the data-splitting method described above, each \(x\) is used to train an equal number of reference models with ($\theta'_x$ ) and without ($\theta'_{\bar{x}}$) it, ensuring an unbiased estimation of \(\Pr(x)\). This setup follows \citet{zarifzadeh2024low} and provides all the necessary information for an optimal online RMIA attack. Together, these attacks provide a progressive understanding of how a model's confidence can reveal membership information.

\begin{proposition} \label{Theorem1}
Let $S_1$ and $S_2$ be two MIA scoring functions:
\begin{align} \label{ine:}
    S_1(x) &= \frac{1}{N} \sum_{j=1}^{N} \mathbf{1}\Bigl( \Pr(x \mid \theta) \geq \Pr(z_j \mid \theta) \Bigr), \\
    S_2(x) &= \Pr(x \mid \theta),
\end{align}
where $\{z_j\}_{j=1}^{N}$ is a fixed set of samples. Suppose $x_1, x_2 \in D$. Then, $S_1$ and $S_2$ are equivalent, meaning that, $S_2(x_1) > S_2(x_2)$, if and only if  $S_1(x_1) > S_1(x_2)$, and their resulting ROC curves are similar.
\end{proposition}

\subsection{Dropout-Enhanced MIA Method: Introducing Prediction Stochasticity} 

Numerous existing works have demonstrated the robustness of SNNs against adversaries due to their discrete spiking behavior \citep{kim2022privatesnn,sharmin2019comprehensive}. Similarly, we expect MIAs to be ineffective, especially on low-latency Spiking SNNs (also validated by our results of Section \ref{sec:results}). It is not because SNNs fail to remember the training data, but because their discrete representations cause member and non-member samples to overlap. Interestingly, we observe that MIA and the out-of-distribution (OOD) detection task share the same principle: 
the model exhibits higher confidence on in-distribution data (members) compared to the out-of-distribution data (non-members) \citep{hendrycks17baseline}. Fundamentally, both methods are related to epistemic uncertainty, which arises due to lack of knowledge or data, and can be reduced with additional information or improved modeling \citep{der2009aleatory,lahlou2023deup}. 

In practice, dropout is a widely used and effective technique for regularization \citep{srivastava2014dropout}, uncertainty estimation \citep{gal2016dropout}, \citep{sun2023efficient}, and OOD detection \citep{hendrycks17baseline, nguyen2022out}. Techniques like Monte Carlo (MC)-Dropout \citep{gal2016dropout} introduce stochasticity during inference, with variations in output softmax probabilities approximating model uncertainty. In the context of MIA, directly applying MC-Dropout is not feasible, as it requires access to model parameters, which is unavailable in a black-box setting. To address this, we propose using input dropout, as an approximation of MC-Dropout (in the first layer), introducing stochasticity to the predictions without requiring access to the model's internals. As evidenced by our experiments, low-latency SNNs often struggle to differentiate subtle distinctions between member and non-member data. By perturbing the input through dropout, we aim to reduce the confidence of non-member data while maintaining the relative confidence of member data, thus enhancing the separation and improving the MIA's effectiveness.

\paragraph{Similarities between weight and input dropout:} Consider a simple single-layered model given by $f_W: \mathbb{R}^d \to \mathbb{R}, \quad f_W(X) = \varphi(W^T \cdot X)$, where $W \in \mathbb{R}^d$ is the weight vector, $X \in \mathbb{R}^d$ represents the input and $\varphi$ is the non-linear activation. During dropout, a Bernoulli mask \( M \in \mathbb{R}^d \) with probability \( (1 - p) \) modifies the input and weight as \( \tilde{X} = M \odot X \) and \( \tilde{W} = M \odot W \), where \( \odot \) denotes the Hadamard product. The equivalence between the expected outputs of input dropout $\mathbb{E}_M[f_W(\tilde{X})]$ and weight dropout $\mathbb{E}_M[f_{\tilde{W}}(X)]$ is shown in Equation \ref{eq: input_drop}:
\begin{equation}
\label{eq: input_drop}
\small
\centering
\begin{split}
    \mathbb{E}_M[f_W(\tilde{X})] = \mathbb{E}_M \left[ \varphi \left( W^T \cdot (M \odot X) \right) \right] \\ 
    = \mathbb{E}_M \left[\varphi\left( \sum_i w_i (M_i x_i) \right)\right]
    = \mathbb{E}_M \left[\varphi\left( \sum_i (M_i w_i) x_i \right)\right] \\ 
    = \mathbb{E}_M \left[ \varphi \left( (M \odot W)^T \cdot X \right) \right] = \mathbb{E}_M[f_{\tilde{W}}(X)]
\end{split}
\end{equation}

However, in case of multi-layered networks, weight dropout is applied to each layer individually, which differentiates it from input dropout. Nevertheless, both methods are capable of introducing stochasticity into the model's predictions. Hence, based on our hypothesis, instead of modifying the internal layers, we randomly drop parts of the input data. For each forward pass \( i \in \{1, \dots, N\} \), a binary mask with elements sampled from the Bernoulli (p) mask inputs: $\tilde{X
} = X \odot M$ is fed into the model \( \theta \). We then compute the confidence \( \Pr(\tilde{X} \mid \theta) \) for each pass, and the final estimate is obtained by averaging over \( n \) passes:

\begin{equation}
    \Pr(X \mid \theta) = \frac{1}{N} \sum_{i=1}^{N} \Pr(\tilde{X}_i \mid \theta).
\end{equation}

This method enhances existing MIA techniques by incorporating the mean confidence value, which can be optimized for different attack strategies. When selecting the hyperparameters, we use the following approach: we randomly consider one of the reference models as the target model and use the other reference models to attack it. Then, we find the optimal hyperparameters \( p \) and \( N \)  by grid search to maximize the attack's AUC. Finally, we use the hyperparameters selected from the reference models to attack the original target model.

Some MIA methods also incorporate input noise \citep{carlini2022membership,zarifzadeh2024low}. However, these methods do not directly estimate the confidence term \( \Pr(X \mid \theta) \); they rely on other components of the MIA metric. Our focus here is not to compare various uncertainty estimation methods but rather to evaluate whether input dropout can significantly enhance MIA performance on SNNs. Additionally, by estimating the confidence term with input dropout, we can extend this approach to all attack methods that require confidence estimation.

\section{Experiments and Results}
\label{sec:exp_res}

\begin{table*}[h!]
\centering
\caption{MIA Results with SNN (left) and ANN (right) ResNet18 on CIFAR-10 and CIFAR-100.}

\begin{minipage}[h]{0.7\textwidth}
\label{tab:mia_results_corrected}
\centering
\footnotesize
\setlength{\tabcolsep}{4.5pt}
\scalebox{0.9}{
\begin{tabular}{%
>{\centering\arraybackslash}m{0.4cm} 
  @{\hspace{-2mm}}              
  >{\centering\arraybackslash}m{1cm} 
  @{\hspace{6pt}}l@{\hspace{6pt}}ccc@{\hspace{6pt}}ccc@{\hspace{6pt}}ccc}     
\toprule
 & \textbf{Drop} & \textbf{Attack} 
& \multicolumn{3}{c}{\textbf{SNN (T=1)}} 
& \multicolumn{3}{c}{\textbf{SNN (T=2)}} 
& \multicolumn{3}{c}{\textbf{SNN (T=4)}} \\ 
\cmidrule(lr){4-6} \cmidrule(lr){7-9} \cmidrule(lr){10-12}
 & \textbf{Out} &  & AUC & 0.1\% & 1\% 
& AUC & 0.1\% & 1\% 
& AUC & 0.1\% & 1\%  \\ 
\midrule
\multirow{6}{*}{\rotatebox{90}{\textbf{CIFAR-10}}}%
& \multirow{3}{*}{\rotatebox{90}{\textbf{False}}}
& Attack-P 
& 54.58 & 0.08 & 0.87 
& 54.90 & 0.00 & 0.72 
& 55.34 & 0.00 & 0.00 \\
&  
& Attack-R 
& 57.59 & 0.00 & 0.00 
& 58.22 & 0.00 & 0.00 
& 59.14 & 0.00 & 0.00 \\
&  
& {\cellcolor{teal!15}}RMIA 
& {\cellcolor{teal!15}}59.89 & {\cellcolor{teal!15}}0.84 & {\cellcolor{teal!15}}3.82 
& {\cellcolor{teal!15}}60.83 & {\cellcolor{teal!15}}1.00 & {\cellcolor{teal!15}}4.43 
& {\cellcolor{teal!15}}62.61 & {\cellcolor{teal!15}}1.09 & {\cellcolor{teal!15}}5.01 \\
\cmidrule(lr){2-12}
& \multirow{3}{*}{\rotatebox{90}{\textbf{True}}}
& Attack-P 
& 54.64 & 0.06 & 0.84 
& 54.75 & 0.05 & 0.83 
& 55.08 & 0.00 & 0.70 \\
&  
& Attack-R 
& 59.54 & 0.00 & 0.00 
& 59.99 & 0.00 & 0.00 
& 60.21 & 0.00 & 0.00 \\
&  
& {\cellcolor{teal!15}}RMIA 
& {\cellcolor{teal!15}}\textbf{63.84} & {\cellcolor{teal!15}}\textbf{1.86} & {\cellcolor{teal!15}}\textbf{6.34} 
& {\cellcolor{teal!15}}\textbf{64.28} & {\cellcolor{teal!15}}\textbf{2.03} & {\cellcolor{teal!15}}\textbf{6.23} 
& {\cellcolor{teal!15}}\textbf{64.65} & {\cellcolor{teal!15}}\textbf{1.94} & {\cellcolor{teal!15}}\textbf{6.61} \\
\midrule
\multirow{6}{*}{\rotatebox{90}{\textbf{CIFAR-100}}}%
& \multirow{3}{*}{\rotatebox{90}{\textbf{False}}}
& Attack-P 
& 59.47 & 0.14 & 1.24 
& 61.67 & 0.12 & 1.13 
& 64.85 & 0.12 & 1.13 \\
&  
& Attack-R 
& 66.25 & 0.00 & 0.00 
& 69.58 & 0.00 & 0.00 
& 73.23 & 0.00 & 0.00 \\
&  
& {\cellcolor{teal!15}}RMIA 
& {\cellcolor{teal!15}}69.06 & {\cellcolor{teal!15}}1.11 & {\cellcolor{teal!15}}6.29 
& {\cellcolor{teal!15}}72.81 & {\cellcolor{teal!15}}1.87 & {\cellcolor{teal!15}}8.24 
& {\cellcolor{teal!15}}77.57 & {\cellcolor{teal!15}}2.61 & {\cellcolor{teal!15}}12.01 \\
\cmidrule(lr){2-12}
& \multirow{3}{*}{\rotatebox{90}{\textbf{True}}}
& Attack-P 
& 59.70 & 0.14 & 1.20 
& 61.85 & 0.13 & 1.30 
& 64.82 & 0.12 & 1.21 \\
&  
& Attack-R 
& 71.69 & 0.00 & 0.00 
& 73.64 & 0.00 & 0.00 
& 75.47 & 0.00 & 0.00 \\
&  
& {\cellcolor{teal!15}}RMIA 
& {\cellcolor{teal!15}}\textbf{75.82} & {\cellcolor{teal!15}}\textbf{2.80} & {\cellcolor{teal!15}}\textbf{11.14} 
& {\cellcolor{teal!15}}\textbf{78.31} & {\cellcolor{teal!15}}\textbf{3.55} & {\cellcolor{teal!15}}\textbf{13.06} 
& {\cellcolor{teal!15}}\textbf{80.76} & {\cellcolor{teal!15}}\textbf{4.36} & {\cellcolor{teal!15}}\textbf{16.16} \\
\bottomrule
\end{tabular}
}
\end{minipage}%
\begin{minipage}[h]{0.28\textwidth}
\centering
\footnotesize
\setlength{\tabcolsep}{3.8pt}
\renewcommand{\arraystretch}{1.5}
\scalebox{0.9}{
\begin{tabular}{@{}c c c c c@{}}
\toprule
& \textbf{Attack} & \multicolumn{3}{c}{\textbf{ANN}}\\
\cmidrule(lr){3-5}
 &  & AUC & 0.1\% & 1\% \\
\midrule
\multirow{3}{*}{\rotatebox{90}{\textbf{CIFAR-10}}} & Attack-P & 56.61 & 0.00 & 0.00 \\
 & Attack-R & 59.49 & 0.00 & 0.00 \\
 & {\cellcolor{teal!15}}RMIA & {\cellcolor{teal!15}}\textbf{63.84} & {\cellcolor{teal!15}}\textbf{2.84} & {\cellcolor{teal!15}}\textbf{5.82}  \\
\midrule
\multirow{3}{*}{\rotatebox{90}{\textbf{CIFAR-100}}} & Attack-P &  68.75 & 0.10 & 1.22 \\
 & Attack-R  & 76.93 & 0.00 & 0.00 \\
 & {\cellcolor{teal!15}}RMIA & {\cellcolor{teal!15}}\textbf{82.51} & {\cellcolor{teal!15}}\textbf{6.75} & {\cellcolor{teal!15}}\textbf{20.25} \\
\bottomrule
\end{tabular}
}
\end{minipage}
\end{table*}

\begin{figure*}[t]
    \begin{minipage}{0.3\textwidth}
        \centering
        \includegraphics[width=\textwidth]{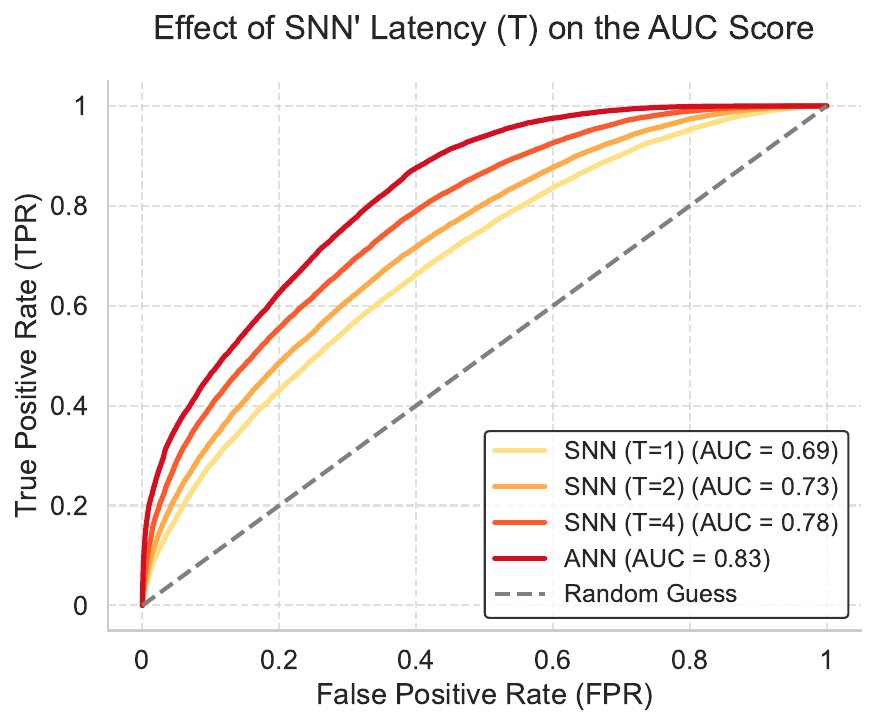}
        \label{fig:plot1}
    \end{minipage}%
    \hspace{0.2cm} 
    \begin{minipage}{0.38\textwidth}
        \centering
        \includegraphics[width=\textwidth]{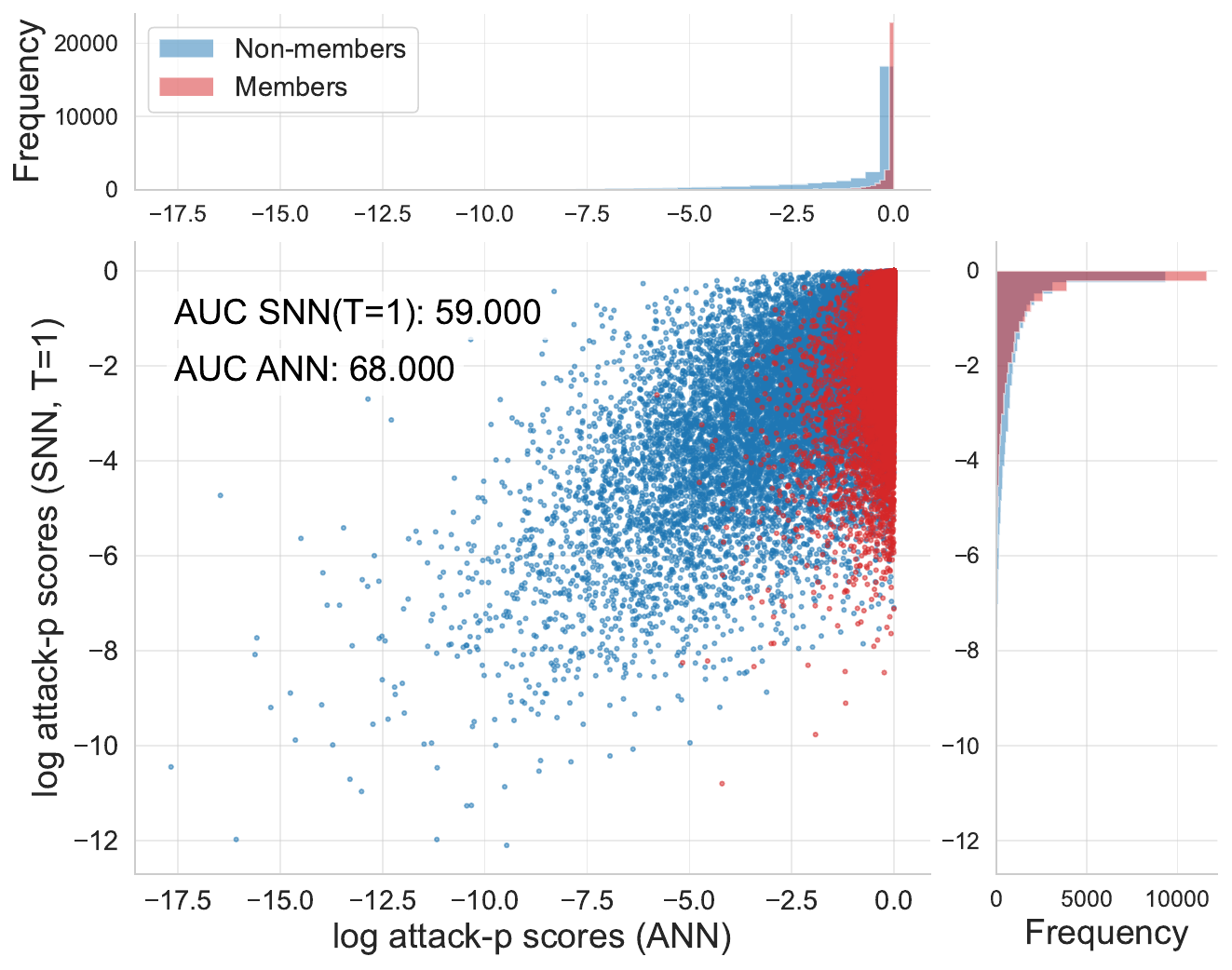}
        \label{fig:plot2}
    \end{minipage}%
    \hspace{0.2cm} 
    \begin{minipage}{0.3\textwidth}
        \includegraphics[width=\textwidth]{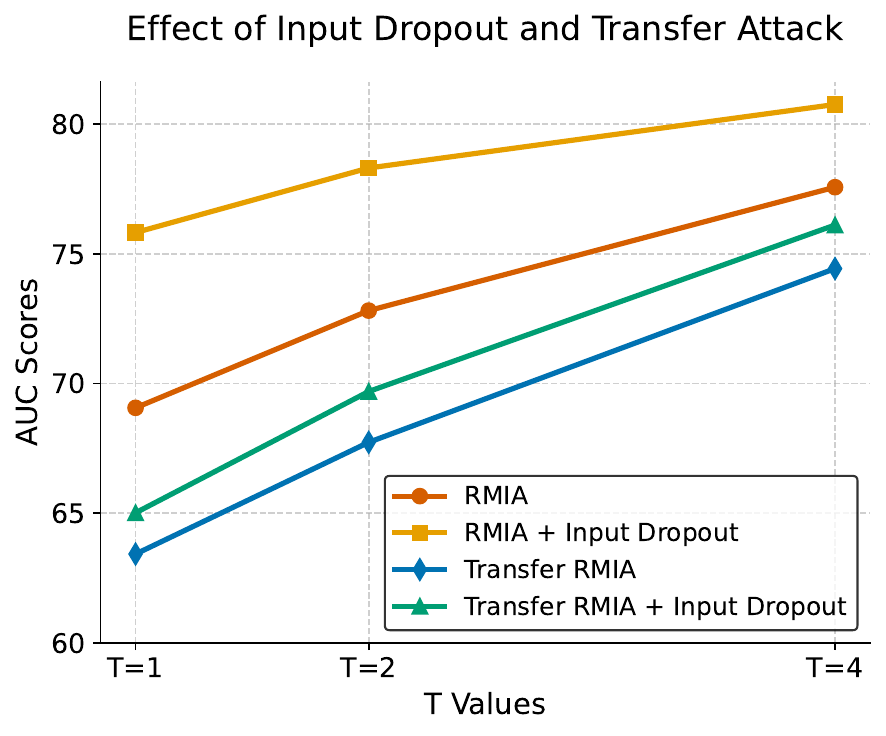}
        \label{fig:plot3}
    \end{minipage}
    
    \caption{ \textbf{Left}: ROC curves and AUC values for all-latency SNNs and an ANN under RMIA on CIFAR-100. \textbf{Middle}: A scatter plot showing the log attack p-scores and distributions for member and non-member data on SNN (T=1) and ANN for CIFAR-100. \textbf{Right}: RMIA performance for all-latency SNNs under different conditions: with or without input dropout, and using or not using a transfer model (ANN).}
    \label{fig:all_plots}
\end{figure*}

The sections describe the detailed outline of our experimental setup and discussion of our results. 

\subsection{Model Training}

\subsubsection{Hybrid Training via ANN-to-SNN Conversion}

In this approach, we first train an ANN and subsequently convert it into SNN using the threshold extraction and fine-tuning method proposed by \citet{bojkovic2024data}. This conversion facilitates efficient SNN training while maintaining high performance, even under low-latency constraints. We consider both CIFAR-10 and CIFAR-100 datasets \citep{krizhevsky2009learning}, where CIFAR-10 represents a relatively simpler classification task, whereas CIFAR-100 presents a more challenging scenario due to its increased class diversity. To ensure a comprehensive evaluation, we train ResNet18 \citep{he2016deep} on both the datasets. 

All models are trained for 200 epochs. For CIFAR-10, we initialize the learning rate at 0.1, while for CIFAR-100, we use 0.02. Optimization is performed using the stochastic gradient descent (SGD) algorithm with a weight decay of $5 \times 10^{-4}$ and a momentum parameter of 0.1. Once the ANN training is complete, we convert the models into SNNs. Following \citet{bojkovic2024data}, we begin by training the converted SNN models at $T = 1$ for 50 epochs using surrogate gradient learning (with a learning rate of $5 \times 10^{-3}$). Higher latency SNN models ($T = 2, 3, 4$) are then obtained through sequential training, where the weights of the SNN model at latency $T$ are initialized from the previously trained SNN at $T - 1$. Each additional latency level is trained for 30 epochs. Table \ref{tab:hybrid_vs_direct_training} shows the test accuracy of trained models.

\begin{table}[h]
\centering
\caption{AUC results for the SNN model (left) and ANN model (right) on CIFAR-10.}
\label{tab:SNN_ANN_combined}

\begin{minipage}{0.68\linewidth}
\centering
\footnotesize
\setlength{\tabcolsep}{6pt}
\renewcommand{\arraystretch}{1.3}
\scalebox{0.86}{
\begin{tabular}{@{}c c c c c@{}}
\toprule
\textbf{Dropout} & \textbf{Attack} & \textbf{T=1} & \textbf{T=2} & \textbf{T=4} \\
\midrule
\multirow{3}{*}{\textbf{False}} 
& Attack-P & 55.13 & 55.57 & 55.95 \\
& Attack-R & 58.33 & 59.90 & 60.25 \\
& \cellcolor{teal!15}RMIA & \cellcolor{teal!15}61.44 & \cellcolor{teal!15}61.97 & \cellcolor{teal!15}63.06 \\
\midrule
\multirow{3}{*}{\textbf{True}} 
& Attack-P & 55.98 & 56.01 & 56.22 \\
& Attack-R & 60.97 & 60.73 & 61.07 \\
& \cellcolor{teal!15}RMIA & \cellcolor{teal!15}\textbf{64.79} & \cellcolor{teal!15}\textbf{65.12} & \cellcolor{teal!15}\textbf{65.36} \\
\bottomrule
\end{tabular}
}
\end{minipage}
\hfill
\begin{minipage}{0.30\linewidth}
\vspace{1pt}
\centering
\footnotesize
\setlength{\tabcolsep}{4.5pt}
\renewcommand{\arraystretch}{1.5}
\scalebox{0.85}{
\begin{tabular}{@{}c c@{}}
\toprule
\textbf{Attack} & \textbf{AUC} \\
\midrule
Attack-P & 56.61 \\
Attack-R & 59.49 \\
\cellcolor{teal!15}RMIA & \cellcolor{teal!15}\textbf{63.84} \\
\bottomrule
\end{tabular}
}
\end{minipage}
\end{table}

To demonstrate the generalizability of our method beyond ResNet, we also provide a brief set of experimental results using both Spiking Neural Networks (SNNs) and Artificial Neural Networks (ANNs). Table~\ref{tab:SNN_ANN_combined} reports the AUC scores for an SNN and ANN model based on VGG-16 architecture on the CIFAR-10 dataset, highlighting the impact of incorporating dropout and varying the number of time steps.

\begin{table*}[t!]
\centering
\caption{MIA results on SNNs with varying latencies (T=1,2,4) using ANNs as reference models, with and without dropout}
\label{tab:ann_to_attack_snn}
\footnotesize
\renewcommand{\arraystretch}{1.2}
\setlength{\tabcolsep}{4pt}
\scalebox{0.88}{
\begin{tabular}{@{}c c l ccc ccc ccc ccc ccc ccc@{}}
\toprule
& & \multicolumn{9}{c}{\textbf{Dropout = False}} & \multicolumn{9}{c}{\textbf{Dropout = True}} \\
\cmidrule(lr){3-11} \cmidrule(lr){12-20}
\textbf{Data} & \textbf{Attack} & \multicolumn{3}{c}{\textbf{SNN(T=1)}} & \multicolumn{3}{c}{\textbf{SNN(T=2)}} & \multicolumn{3}{c}{\textbf{SNN(T=4)}} & \multicolumn{3}{c}{\textbf{SNN(T=1)}} & \multicolumn{3}{c}{\textbf{SNN(T=2)}} & \multicolumn{3}{c}{\textbf{SNN(T=4)}} \\
\cmidrule(lr){3-5}\cmidrule(lr){6-8}\cmidrule(lr){9-11}\cmidrule(lr){12-14}\cmidrule(lr){15-17}\cmidrule(lr){18-20}
& & AUC & 0.1\% & 1\% & AUC & 0.1\% & 1\% & AUC & 0.1\% & 1\% & AUC & 0.1\% & 1\% & AUC & 0.1\% & 1\% & AUC & 0.1\% & 1\% \\
\midrule

\textbf{CIFAR}
  & \textbf{Attack-R}
  & 53.26 & 0.00 & 1.74 & 53.88 & 0.00 & 0.49 & 55.04 & 0.00 & 0.00 
  & 53.31 & 0.25 & 1.56 & 54.10 & 0.00 & 0.50 & 55.50 & 0.00 & 0.00 \\
\textbf{10} & {\cellcolor{teal!15}}\textbf{RMIA}
  & {\cellcolor{teal!15}}58.52 & {\cellcolor{teal!15}}0.69 & {\cellcolor{teal!15}}3.40 
  & {\cellcolor{teal!15}}59.68 & {\cellcolor{teal!15}}\textbf{0.81} & {\cellcolor{teal!15}}3.86 
  & {\cellcolor{teal!15}}61.32 & {\cellcolor{teal!15}}\textbf{1.05} & {\cellcolor{teal!15}}4.55 
  & {\cellcolor{teal!15}}\textbf{59.32} & {\cellcolor{teal!15}}\textbf{0.74} & {\cellcolor{teal!15}}\textbf{3.77} 
  & {\cellcolor{teal!15}}\textbf{60.52} & {\cellcolor{teal!15}}0.76 & {\cellcolor{teal!15}}\textbf{4.34} 
  & {\cellcolor{teal!15}}\textbf{62.29} & {\cellcolor{teal!15}}0.89 & {\cellcolor{teal!15}}\textbf{4.98} \\

\midrule

\textbf{CIFAR}
  & \textbf{Attack-R}
  & 58.99 & 0.00 & 0.80 & 62.71 & 0.00 & 1.89 & 67.98 & 0.00 & 0.00 
  & 59.41 & 0.00 & 0.49 & 63.31 & 0.00 & 1.40 & 68.58 & 0.00 & 0.00 \\
\textbf{100} & {\cellcolor{teal!15}}\textbf{RMIA}
  & {\cellcolor{teal!15}}63.42 & {\cellcolor{teal!15}}\textbf{0.62} & {\cellcolor{teal!15}}3.85 
  & {\cellcolor{teal!15}}67.73 & {\cellcolor{teal!15}}\textbf{1.02} & {\cellcolor{teal!15}}\textbf{6.29} 
  & {\cellcolor{teal!15}}74.43 & {\cellcolor{teal!15}}\textbf{1.37} & {\cellcolor{teal!15}}9.65 
  & {\cellcolor{teal!15}}\textbf{65.01} & {\cellcolor{teal!15}}0.38 & {\cellcolor{teal!15}}\textbf{3.89} 
  & {\cellcolor{teal!15}}\textbf{69.69} & {\cellcolor{teal!15}}0.72 & {\cellcolor{teal!15}}6.13 
  & {\cellcolor{teal!15}}\textbf{76.11} & {\cellcolor{teal!15}}1.10 & {\cellcolor{teal!15}}\textbf{10.39} \\

\bottomrule
\end{tabular}
}
\end{table*}

\begin{table}[t]
\centering
\caption{Directly trained SNN Attack Results (T=4).}
\label{tab:snn_t4_results}
\renewcommand{\arraystretch}{1.2}
\footnotesize
\setlength{\tabcolsep}{4pt}
\scalebox{0.9}{
\begin{tabular}{@{}c c l ccc@{}}
\toprule
\textbf{Dataset} & \textbf{Dropout} & \textbf{Attack} 
& AUC & 0.1\% & 1\% \\
\midrule

\multirow{6}{*}{\textbf{CIFAR-10}}
  & \multirow{3}{*}{\textbf{False}}
    & Attack-P 
    & 59.76 & 0.05 & 1.05 \\
  &
    & Attack-R 
    & 63.53 & 0.00 & 0.00 \\
  &
    & {\cellcolor{teal!15}}RMIA 
    & {\cellcolor{teal!15}}63.85 & {\cellcolor{teal!15}}1.04 & {\cellcolor{teal!15}}5.76 \\
\cmidrule(lr){2-6}
  & \multirow{3}{*}{\textbf{True}}
    & Attack-P 
    & 58.87 & 0.06 & 0.96 \\
  &
    & Attack-R 
    & 64.64 & 0.00 & 0.00 \\
  &
    & {\cellcolor{teal!15}}RMIA 
    & {\cellcolor{teal!15}}\textbf{67.58} & {\cellcolor{teal!15}}\textbf{3.56} & {\cellcolor{teal!15}}\textbf{8.86} \\
\midrule
\multirow{6}{*}{\textbf{CIFAR-100}}
  & \multirow{3}{*}{\textbf{False}}
    & Attack-P 
    & 56.54 & 0.00 & 0.74 \\
  &
    & Attack-R 
    & 56.15 & 0.00 & 0.00 \\
  &
    & {\cellcolor{teal!15}}RMIA 
    & {\cellcolor{teal!15}}59.62 & {\cellcolor{teal!15}}0.44 & {\cellcolor{teal!15}}2.39 \\
\cmidrule(lr){2-6}
  & \multirow{3}{*}{\textbf{True}}
    & Attack-P 
    & 56.58 & 0.00 & 1.06 \\
  &
    & Attack-R 
    & 57.65 & 0.00 & 0.00 \\
  &
    & {\cellcolor{teal!15}}RMIA 
    & {\cellcolor{teal!15}}\textbf{62.11} & {\cellcolor{teal!15}}\textbf{0.62} & {\cellcolor{teal!15}}\textbf{3.13} \\
\bottomrule
\end{tabular}
}
\end{table}

\subsubsection{Direct SNN Training}

To validate the findings obtained from our hybrid training approach, we conduct an additional set of experiments using directly trained SNNs, following the methodology proposed in \citet{mukhoty2023direct}. For this experiment also, we focus on the CIFAR-10 and CIFAR-100 dataset and train a ResNet18 SNN model. The model is trained for 250 epochs using the Adam optimizer with a learning rate of 0.001. We evaluate the SNN at $T=4$ to assess its impact on privacy risks. Unlike the sequential training strategy employed in the hybrid approach, direct SNN training allows for parallel training of models at different latencies, as each latency level is independently optimized. To ensure fair comparison, we apply the same data preprocessing and transformations as in the hybrid training setup. All training and experiments are conducted on NVIDIA A100 GPUs.

\begin{table}[H]
    \centering
    \small
    \setlength{\tabcolsep}{3pt}
    \renewcommand{\arraystretch}{1.2}
    \caption{ResNet18 classification test accuracy (\%) for ANN-to-SNN converted models and directly trained SNNs.}
    \label{tab:hybrid_vs_direct_training}
    \scalebox{0.79}{
    \begin{tabular}{l c c c c c c}
        \toprule
         & \multicolumn{3}{c}{\textbf{Hybrid}} & \multicolumn{1}{c}{\textbf{Direct}} & {\cellcolor{teal!15}}\textbf{ANN} \\
        \cmidrule(lr){2-4} \cmidrule(lr){5-5} \cmidrule(lr){6-6} 
        \textbf{Dataset} & \textbf{SNN (T=1)} & \textbf{SNN (T=2)} & \textbf{SNN (T=4)} & \textbf{SNN (T=4)} & {\cellcolor{teal!15}}\textbf{--}\\
        \midrule
        \textbf{CIFAR-10}  & 86.1 & 87.3 & 88.2 & 87.0 & {\cellcolor{teal!15}} \textbf{90.9}\\
        \midrule
        \textbf{CIFAR-100} & 59.5 & 62.1 & 63.9 & 66.9 & {\cellcolor{teal!15}} \textbf{69.5}\\
        \bottomrule
    \end{tabular}
    }
\end{table}

\label{sec:results}

\subsection{SNN Performance against MIA} 



Tables \ref{tab:mia_results_corrected} and Figure \ref{fig:all_plots} (\textbf{Left}) summarize the performance of trained ANNs and their corresponding SNNs ($T=1,2,4$) on CIFAR-10 and CIFAR-100 when subjected to MIAs. Our findings highlight three key insights:

\textbf{Effectiveness of Attacks:} RMIA consistently outperforms both Attack-R and Attack-P, confirming its superiority as a membership inference technique. This aligns with prior studies \citep{zarifzadeh2024low}, reinforcing RMIA's superior privacy threat in ANNs and SNNs.

\textbf{Impact of Latency on SNN Vulnerability:} The privacy risk of SNNs increases with latency ($T$), across all attack metrics. As the T increases, the SNN behaves similar to an ANN and hence, the difference in membership scores of member and non-member data widens.

\textbf{ANNs vs. SNNs in MIA Susceptibility:} ANNs consistently exhibit the highest attack metric scores, indicating greater susceptibility to MIAs compared to SNNs. This suggests that SNNs offer an intrinsic advantage in privacy protection, particularly at lower latencies.

\textbf{Comparing directly trained SNNs:} Table \ref{tab:snn_t4_results} highlights the MIA threat to directly trained SNNs (T=4), showing AUC trends similar to hybrid training. This confirms that SNN vulnerability to MIA is independent of the training method.

The resilience of SNNs at low $T$ stems from their discrete processing. Unlike ANNs, which use continuous values, SNNs encode information as spike trains over time steps. When $T=1$, each neuron outputs a single binary value, significantly limiting expressiveness and leading to overlapping confidence distributions for member and non-member data. This overlap reduces the attacker's ability to distinguish between them, thereby mitigating MIA effectiveness (see Fig. \ref{fig:all_plots}, \textbf{Middle}). However, as $T$ increases, SNNs accumulate more information over time, improving expressiveness and decision boundary separation (Fig. \ref{app:forhis}), which in turn makes them more vulnerable to MIAs. This reveals an accuracy-privacy trade-off: lower $T$ enhances privacy but slightly reduces accuracy, while higher $T$ improves accuracy at the cost of privacy (Table \ref{tab:hybrid_vs_direct_training}).

\subsection{Attacking SNN Using ANN as Reference Models}

SNNs consists of critical hyperparameters, such as latency (T) and membrane threshold, which significantly influence their performance. Unlike ANNs, these parameters are not easily accessible to an adversary, adding an inherent layer of complexity for the attack. Consequently, an attacker may resort to using a conventionally trained ANN with the same architecture as a reference model to infer membership on the target SNN. As expected, our experiments indicate that this approach is less effective (Fig. \ref{fig:all_plots}, \textbf{Right}).

Although there are architectural similarities, the confidence distributions of ANNs and SNNs exhibit notable differences, particularly in how they separate member and non-member samples (Fig. \ref{fig:all_plots}, \textbf{Middle}). Despite this, we observe a pattern similar to that of the same-model scenario in Table \ref{tab:ann_to_attack_snn}. Furthermore, the trends in attack metrics for RMIA and Attack-R across both datasets closely mirror those in the same-model setting. At lower SNN latencies, the confidence scores for member and non-member data become increasingly indistinguishable compared to ANNs (Fig. \ref{app:forhis}). 

\begin{figure}[h!]
    \centering
    \includegraphics[scale=0.27]{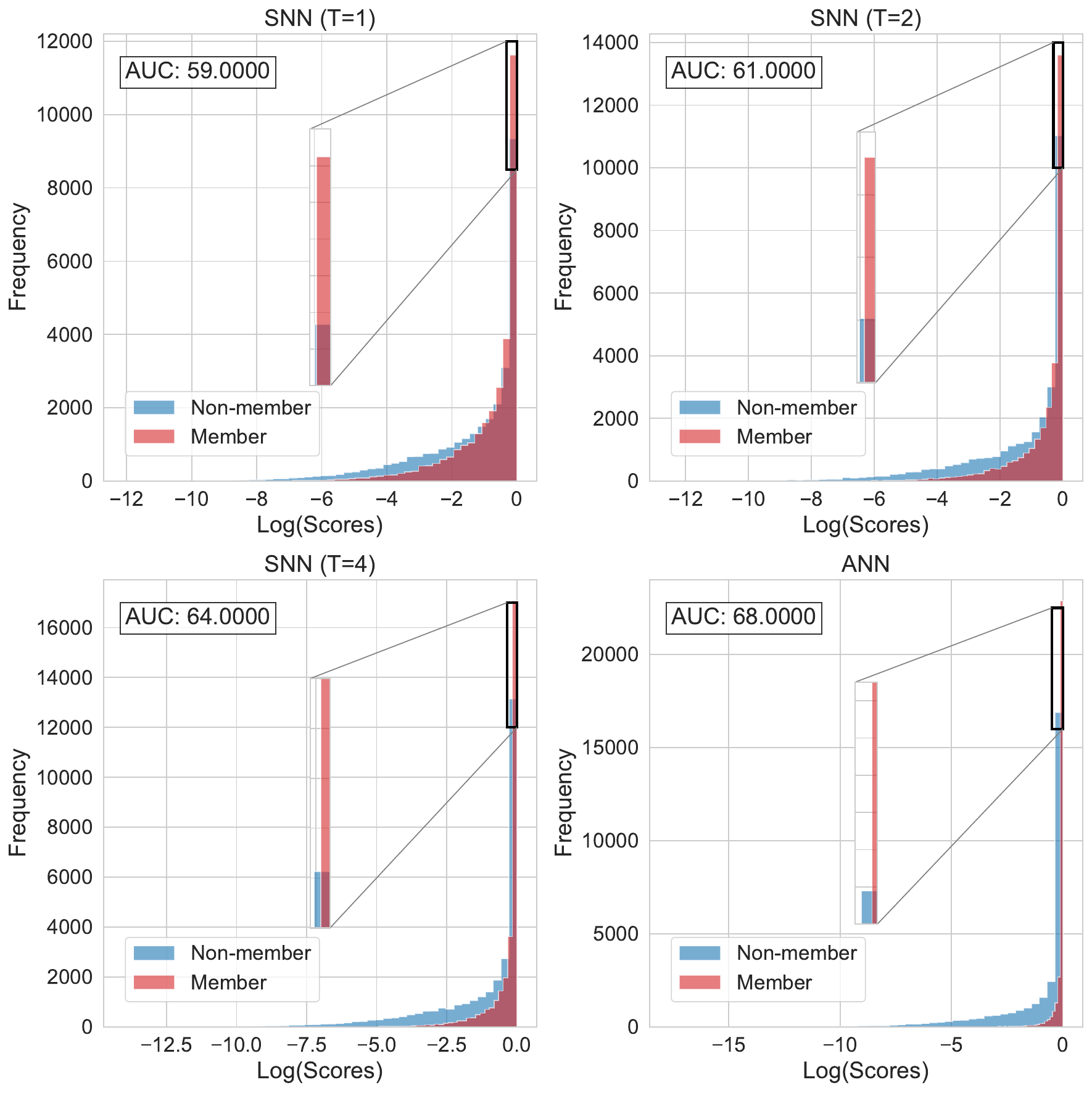}  
    \caption{
    Attack-P(confidence) distribution of member and non-member for all latency SNNs \& ANNs in CIFAR-100}
    \label{app:forhis}
\end{figure}

\subsection{Impact of Input Dropout on Attack Performance}

To enhance the attack on SNNs, we apply the input dropout method. For both CIFAR-10 and CIFAR-100, our experiments demonstrate that incorporating input dropout significantly increases the overall attack performance. We observe that the vulnerability of SNNs becomes quite similar to that of ANNs for all attacks as the input dropout strategy is adopted. This result is clearly evident for the most potent RMIA attack on both datasets. 

Interestingly, the AUC improvement also occurs even when the adversary uses ANNs as reference models to attack the target SNN model, as shown in Table \ref{tab:ann_to_attack_snn} and Figure \ref{fig:all_plots}. This highlights the transferability of our hypothesis regarding SNN vulnerabilities, reinforcing the concern that adversarial threats can persist even without direct knowledge of the model's structure. Furthermore, similar enhancements in attack performance are also evident for directly trained SNNs, particularly for low-latency configurations as shown in Table \ref{tab:snn_t4_results}. This shows that the technique adopted to train SNNs is inconsequential to the MIA threat and our input dropout strategy, highlighting the universal susceptibility of SNN to the privacy threat.

\section{Conclusion}
\label{sec:conc}
Our results show that SNNs exhibit greater resilience to MIAs than ANNs. However, it becomes increasingly vulnerable as the latency increases. Furthermore, our proposed naive input dropout technique can significantly increase the risk of MIA, even when the adversary lacks knowledge of the SNN's architecture and uses ANN as reference models. This effect also persists in SNNs trained directly. Therefore, we argue that the assumption of SNNs' inherent privacy protection should be critically evaluated, especially for sensitive tasks. Our work also highlights the need for robust techniques to enhance privacy assurance in SNNs.

\begin{acknowledgements}
The authors express their sincere gratitude to Velibor Bojkovic, Ruben Solozabal, and Xingyu Qu for their valuable insights and constructive feedback, which greatly contributed to the improvement of this work.
\end{acknowledgements}

\bibliography{uai2025-template}

\newpage

\onecolumn

\title{On the Privacy Risks of Spiking Neural Networks: \\A Membership Inference Analysis\\(Supplementary Material)}
\maketitle

\appendix
\section{Proofs}

\textbf{Proposition 1: }\emph{
Let $S_1$ and $S_2$ be two MIA scoring functions:
\begin{align} \label{ine:}
    S_1(x) &= \frac{1}{N} \sum_{j=1}^{N} \mathbf{1}\Bigl( \Pr(x \mid \theta) \geq \Pr(z_j \mid \theta) \Bigr), \\
    S_2(x) &= \Pr(x \mid \theta),
\end{align}
where $\{z_j\}_{j=1}^{N}$ is a fixed set of samples. Suppose $x_1, x_2 \in D$. Then, $S_1$ and $S_2$ are equivalent, meaning that, $S_2(x_1) > S_2(x_2)$, if and only if  $S_1(x_1) > S_1(x_2)$, and their resulting ROC curves are similar.
}

\begin{proof}
Assume $S_2(x_1) > S_2(x_2)$, meaning
\begin{equation}
    \Pr(x_1 \mid \theta) > \Pr(x_2 \mid \theta).
\end{equation}
Define the indicator set:
\begin{equation}
    A_x = \{ z_j \mid \Pr(x \mid \theta) \geq \Pr(z_j \mid \theta) \}.
\end{equation}
Since $\Pr(x_1 \mid \theta) > \Pr(x_2 \mid \theta)$, we have $A_{x_2} \subset A_{x_1}$. Hence,
\begin{equation}
    S_1(x_1) = \frac{|A_{x_1}|}{N} > \frac{|A_{x_2}|}{N} = S_1(x_2).
\end{equation}
This establishes the forward order-preserving property.

\textbf{Reverse direction:} Assume $S_1(x_1) > S_1(x_2)$. This implies
\begin{equation}
    \frac{|A_{x_1}|}{N} > \frac{|A_{x_2}|}{N}.
\end{equation}
Since $A_x$ is defined based on comparisons with $\Pr(x \mid \theta)$, if $S_1(x_1) > S_1(x_2)$, then $x_1$ ranks strictly higher than $x_2$ in terms of how often it exceeds reference points $z_j$. The only way for this to happen is if
\begin{equation}
    \Pr(x_1 \mid \theta) > \Pr(x_2 \mid \theta),
\end{equation}
which implies $S_2(x_1) > S_2(x_2)$. Since both directions hold, we conclude
\begin{equation}
    S_2(x_1) > S_2(x_2) \iff S_1(x_1) > S_1(x_2).
\end{equation}
Hence the $S_1$ and $S_2$ are equivalent. By Lemma \ref{theorem2}, the ROC for them are the same 
\end{proof}

\begin{lemma}
    \label{theorem2}
    Let $S_1$ and $S_2$ be two equivalent scoring functions, meaning that they satisfy, for every $x_1, x_2$: 
    \begin{equation}
    S_2(x_1) > S_2(x_2) \iff S_1(x_1) > S_1(x_2).
\end{equation}
Then the ROC curves for $S_1$ and $S_2$ on the same dataset are identical.
\end{lemma}

\begin{proof}
    Since $S_1$ and $S_2$ preserve the same ranking of data points, their ordering in terms of classification thresholds remains unchanged. That is, for any threshold $t$ applied to $S_2$, there exists a corresponding threshold $s$ applied to $S_1$ such that the classification outcomes remain identical:
\begin{equation}
    \hat{y}_i = \begin{cases} 1, & S_2(x_i) \geq t \\ 0, & S_2(x_i) < t \end{cases} \iff \hat{y}_i = \begin{cases} 1, & S_1(x_i) \geq s \\ 0, & S_1(x_i) < s. \end{cases}
\end{equation}
Since both functions yield the same classification outcomes for all possible threshold values, they result in the same TPR and FPR values at each threshold:
\begin{equation}
    \text{TPR}_{S_1}(s) = \text{TPR}_{S_2}(t), ~~~~   \text{FPR}_{S_1}(s) = \text{FPR}_{S_2}(t).
\end{equation}
Since the ROC curve is defined as the parametric plot of $(\text{FPR}(t), \text{TPR}(t))$, it follows that:
\begin{equation}
    \text{ROC}_{S_1} = \text{ROC}_{S_2}.
\end{equation}
\end{proof}

\end{document}